\documentclass{article}
\usepackage[english]{babel}
\usepackage[normalem]{ulem}
\usepackage{setspace}

\usepackage[letterpaper,top=2cm,bottom=2cm,left=3cm,right=3cm,marginparwidth=1.75cm]{geometry}

\usepackage{amsthm}
\usepackage{amsmath,amssymb}
\theoremstyle{plain} 
\newtheorem{theorem}{Theorem}
\newtheorem{proposition}{Proposition}
\theoremstyle{remark}
\theoremstyle{definition}

\usepackage{graphicx}
\usepackage{float}
\usepackage[colorlinks=true, allcolors=blue]{hyperref}
\usepackage{subcaption}
\newcommand{\R}{{\mathbb R}}
\newcommand{\eps}{{\varepsilon}}
\newcommand{\calL}{\mathcal L}

\title{\textbf{Sampling Boltzmann distributions via normalizing flow approximation of transport maps}}

\author{%
	Zia Ur Rehman and Gero Friesecke\\
	\small Department of Mathematics, School of Computation, Information and Technology,\\
	\small Technical University of Munich, Boltzmannstrasse 3, 85748 Garching, Germany\\
	\small \texttt{ ziaur.rehman@tum.de, gf@ma.tum.de}
}

\begin{document}
	\maketitle
	
	\begin{abstract}
		In a celebrated paper \cite{noe2019boltzmann}, Noé, Olsson, K\"ohler and Wu introduced an efficient method for sampling high-dimensional Boltzmann distributions arising in molecular dynamics via normalizing flow approximation of transport maps. Here, we place this approach on a firm mathematical foundation. We prove the existence of a normalizing flow between the reference measure and the true Boltzmann distribution up to an arbitrarily small error in the Wasserstein distance. This result covers general Boltzmann distributions from molecular dynamics, which have low regularity due to the presence of interatomic Coulomb and Lennard-Jones interactions. The proof is based on a rigorous construction of the Moser transport map for low-regularity endpoint densities and approximation theorems for neural networks in Sobolev spaces. 
		
		Numerical simulations for a simple model system and for the alanine dipeptide molecule confirm that the true and generated distributions are close in the Wasserstein distance. Moreover we observe that the RealNVP architecture does not just successfully capture the equilibrium Boltzmann distribution but also the metastable dynamics.

	\end{abstract}
	
	\large
	
	\section{Introduction}
	
	Our goal in this paper is to put the recent method of No\'{e} et al. \cite{noe2019boltzmann} for sampling from high-dimensional Boltzmann distributions on a firm mathematical footing. 
	
	Sampling from Boltzmann distributions is of great interest in computational physics, chemistry, and biochemistry, because the conformations of a molecule are described by the dominant regions of the associated Boltzmann distribution. Sampling from this distribution is challenging for two reasons: first, the configuration space on which the distribution lives is high-dimensional, of dimension $3N$, where $N$ is the number of atoms in the molecule.  Second, trajectory methods, i.e. using long-time samples of solutions of the underlying Langevin stochastic differential equation, while in principle correct, are expensive; this is because trajectories typically get trapped for a long time in metastable regions separated by energetic or entropic barriers, leading to very long simulation times for sufficient exploration of the configuration space.

	Traditional approaches to this sampling problem include Markov state models \cite{chodera2007automatic}, metadynamics \cite{laio2002escaping}, and umbrella sampling \cite{torrie1977nonphysical}. While useful, these methods often either  require careful hand-tuning of collective variables or suffer from high computational cost in high dimensions.
	
	A novel approach based on ideas from machine learning was introduced in a pioneering paper by No\'e et al.~\cite{noe2019boltzmann}, who proposed \emph{Boltzmann generators}---invertible neural networks (normalizing flows) trained to map an easy-to-sample reference distribution (or prior) on the high-dimensional configuration space, for instance a Gaussian, to the target Boltzmann distribution. Note that, unlike in standard approaches in machine learning for approximating high-dimensional distributions, \textit{the domain of the prior is not dimension-reduced.} Once the flow is trained, a single forward pass through the flow of statistically independent samples drawn from the prior yields statistically independent equilibrium samples. 
	This bypasses long-time stochastic trajectories and enables efficient exploration of metastable states. This approach builds on the rapid development of normalizing flows \cite{dinh2017density,kingma2018glow,papamakarios2021normalizing} and deep generative modeling for physical systems \cite{hernandez2019variational,wu2020deep}.
	
	Despite impressive empirical successes, a mathematical foundation of flow-based Boltzmann sampling is hitherto missing. Several open questions present themselves: 
	\begin{itemize}
		\item \textbf{Dealing with singularities:} Physical Boltzmann distributions vanish on the collision set where two atoms collide, due to the singularities of the underlying Coulomb and Lennard-Jones potentials. Such zero regions, in turn, prevent the existence of smooth transport maps (for a proof of this interesting phenomenon see section \ref{sec:regu}).  
		Thus it is not even clear whether any transport map exists between the Gaussian reference and the Boltzmann distribution which belongs to the regularity classes covered by expressivity results for neural networks.  
		\item \textbf{Indeterminacy:} Even if singularities are absent or have been removed by some controlled approximation, the transport condition is highly degenerate. Which of the infinite-dimensionally many transport maps should be targeted for theoretical approximation results (and practical training)?
		\item \textbf{Approximability by normalizing flows:} Because of the need to efficiently evaluate the Jacobian of the transport map, a suitable class of neural networks is given by normalizing flow architectures such as RealNVP \cite{dinh2017density}. Do these architectures possess sufficient expressive power to approximate the required transport maps between low-regularity densities well enough to produce correct samples? 
	\end{itemize}
	
	The first problem is overcome by a simple regularization of the underlying molecular potential in high-energy regions, which entails a controlled $L^1$ error on the Boltzmann distribution. See section \ref{sec:regu}. 
	
	Regarding the second problem, a natural idea would be to use the Brenier map supplied by optimal transport theory \cite{brenier1991polar, villani2003, friesecke2024optimal}. For image generation, methods based on the Brenier maps from Gaussian noise to subsets of training images have proven useful  \cite{lipmanflowmatching,liu2023rectifiedflow, chemseddine2025, tong2023, pooladian2023}. Neural network based solvers for more general optimal transport problems have been introduced in \cite{korotin2021, uscidda2023mongegap}, based on input convex eural networks (ICNNs) respectively multi-layer perceptrons (MLPs).  But the Brenier map, and the underlying flow, are not easy to compute between continuous distributions in high dimensions, and known regularity results for the map \cite{caffarelli1991regularity,dephilippis2011regularity} require strong assumptions on the densities. We therefore focus on a related but easier-to-compute and more regular transport map between continuous distributions, namely the Dacorogna-Moser map, and its underlying flow. We generalize its classical construction \cite{moser1965, dacorogna1990} from smooth densities to low-regularity densities (see Theorem \ref{thm:moser-approx} a)). 
	We remark that for a uniform reference density, the governing problem correponds to a linearization of the optimal transport problem \cite{friesecke2024optimal}. 
	The practical question of how to efficiently bias the learning of the flow towards the Dacorogna-Moser flow lies beyond the scope of this paper. 
	
	Finally, the third problem, rigorous approximability, is addressed in Theorem  \ref{thm:moser-approx} c).  
	We prove that the Moser map between continuously differentiable, everywhere positive densities can be approximated by normalizing flow architechtures such as RealNVP in the Sobolev space $W^{1,p}$.
	
	Combining these results yields existence of a normalizing flow between the reference measure and the true Boltzmann distribution of a molecular system up to an arbitrarily small error in the Wasserstein distance (see Theorem \ref{thm:Wasserstein}), thus providing a theoretical guarantee that the Boltzmann generator method gives correct samples. 
	
	Numerical simulations for a simple synthetic system and for the alanine dipeptide molecule (see section \ref{sec:double-well} and ~\ref{sec:alanine}) confirm that the learned flows yield correct samples reproducing the equilibrium densities. 
	
	\section{Molecular dynamics at thermal equilibrium}
	
	In equilibrium molecular dynamics, the state of a many-body system can be described by the Boltzmann probability distribution
	\begin{equation} \label{boltz}
		d\rho(x) = \frac{1}{Z} \exp\Bigl(-\frac{U(x)}{\kappa_B T}\Bigr)\mathrm{d}x,
	\end{equation}
	where $x=(x_1,...,x_N)$ with $x_i\in \R^3$ denoting the position of the $i$-th atom,  \(\kappa_B\) is a physical constant, \(T\) is the temperature, and \(U\) is the potential energy function of the system. One assumes that the position vector $x$ of all atoms is confined to some bounded domain $D\subset\R^{3N}$, and $Z$ is a normalization constant so that $\int_D \rho = 1$. 
	
	The potential energy function is typically taken to be of the form
	\begin{equation} \label{pot}
		\begin{aligned}
			U(x) = & \sum_{\ell_{ij} \text{ bond lengths}} b_{ij}(\ell_{ij}-\overline{\ell_{ij}})^2 
			+ \sum_{\alpha_{ijk} \text{ bond angles}} a_{ijk}(\alpha_{ijk}-\overline{\alpha_{ijk}})^2 \\
			& + \sum_{\phi_{ijkl} \text{ torsion angles}} \sum_{n=1}^{n_{max}} \kappa_{ijkl,n}\left[ 1+\cos(n\phi_{ijkl}-\overline{\phi_{ijkl,n}}) \right] \\
			& + \sum_{i,j} \frac{q_i q_j}{|x_i-x_j|} + \sum_{i,j} \left( \frac{A_{ij}}{|x_i-x_j|^{12}} - \frac{B_{ij}}{|x_i-x_j|^6} \right),
		\end{aligned}
	\end{equation}
	where the first three terms are summed, respectively, over the bond lengths, bond angles, and torsion angles in the bond graph of the molecular system and the last two terms are summed over all atoms in the system. In the above \(q_i\), \(A_{ij}\), \(B_{ij}\), \(a_{ijk}\), \(b_{ij}\), \(\kappa_{\phi_{ijkl,n}}\), \(\overline{\ell_{ij}}\), \(\overline{\alpha_{ijk}}\), and \(\overline{\phi_{ijkl,n}}\) are constants specific to the molecular system. The constants $b_{ij}$ and $a_{ijk}$ are positive, $\kappa_{ijkl,n}$, $A_{ij}$ and $B_{ij}$ are nonnegative, and $A_{ij}$ must be positive when $q_iq_j<0$ or $B_{ij}>0$, and \(x_i \in \R^3\) is the position of the \(i\)-th atom in the system.
	The last two terms in the energy functional represent electrostatic and Lennard-Jones interactions. 
	
	The energy \eqref{pot} has various important qualitative properties: 
	\begin{equation} \label{potquali}
		\begin{aligned}
			& U \, : \, \R^{3N}\to\R\cup\{+\infty\} \text{ is $C^1$ on }\{x \, : \, x_i-x_j\neq 0 \; \forall i\neq j \}, \\
			& U(x) \ge U_0 \text{ for some constant }U_0, \\
			& U(x) \to \infty \text{ as }\min_{i\neq j} |x_i-x_j|\to 0.
		\end{aligned}
	\end{equation}
	The last property in \eqref{potquali} means that the potential energy tends to infinity when two atoms approach each other. As a consequence, the Boltzmann distribution \eqref{boltz} vanishes on the collision set $\{x_i=x_j \, \text{for some }i\neq j\}$.

	\section{Transport maps and their lack of regularity} \label{sec:regu} 
	The Boltzmann generator approach for sampling from Boltzmann distributions starts from a smooth, positive, easy-to-sample reference distribution, or prior, and seeks to compute an invertible transport map between the reference distribution and the Boltzmann distribution \eqref{boltz}. We begin by recalling basic mathematical definitions and \textcolor{black}{facts}. We then show that in our present context, transport maps must necessarily be quite irregular.
	
	Assume that we are given two probability measures $\mu$ and $\nu$ on $\R^d$. (In the molecular dynamics case, $d=3N$.) A \textit{transport map} between $\mu$ and $\nu$ is a measurable map $T : \R^d \to \R^d$ satisfying
	\begin{equation} \label{pfw}
		T_\sharp \mu = \nu.
	\end{equation}
	The pushforward $T_\sharp \mu$ is defined such that for any measurable set $A \subset \R^d$, $(T_\sharp \mu)(A) = \mu(T^{-1}(A))$. Equivalently, for any measurable function $f$, we have
	\[
	\int_{\R^d} f(y)\, d(T_\sharp \mu)(y)
	=
	\int_{\R^d} f(T(x))\, d\mu(x).
	\]
	This means that if a random variable $X\sim \mu$ then $T(X)\sim \mu$. 
	
	Assume now that $\mu$ and $\nu$ possess densities $\rho_0$ and $\rho_1$. In this case we write condition \eqref{pfw} as 
	\begin{equation} \label{pfw'} 
		T_\sharp \rho_0=\rho_1.
	\end{equation}
	If the map $T$ is invertible and its inverse is differentiable, by the change-of-variables formula eq.~\eqref{pfw} is equivalent to 
	\begin{equation} \label{pfw_densities}
		\rho_1(z)
		=
		\rho_0(T^{-1}(z)) \left| \det J_{T^{-1}}(z) \right|.
	\end{equation}
	where $J_{T^{-1}}$ is the Jacobian matrix of $T$. 
	
	Invertible transport maps exist (in the low-regularity class of measurable mappings); a basic example is the Brenier map \cite{brenier1991polar, villani2003, friesecke2024optimal} from optimal transport theory which exists provided the densities have finite second moments or compact support. But this map is not easy to compute in high dimension. Moreover known regularity results \cite{caffarelli1991regularity,dephilippis2011regularity} require the target density to be bounded away from zero, which is violated by Boltzmann distributions for molecules, due to \eqref{potquali}.
	
	Another, simpler and smoother, example of a transport map is the Dacorogna-Moser map, described in the next section; but its construction requires, among other assumptions, strict positivity of the target density, and therefore fails for Boltzmann distributions for molecules, which vanish on the collision set, due to \eqref{potquali}.
	
	In fact we have:
	\begin{proposition}[Nonexistence of Lipschitz transport maps] \label{P:nonexist} Let $\Omega\subset\R^{3N}$ be any open bounded set, let $\rho_0$ be any reference probability density on $\Omega$ which is continuous and strictly positive, and let $\rho_1$ be a Boltzmann distribution \eqref{boltz} on an open bounded set $D\subset\R^{3N}$ with $U$ satisfying \eqref{potquali}. Then there does not exist any Lipschitz map $T$ 
		which transports $\rho_0$ to the Boltzmann distribution $\rho_1$.  
	\end{proposition}
	The singular nature of (exact) transport maps revealed by this result is a challenge for theory, and also for efficient learning of such maps, on which more later. 
	\\[3mm]
	\begin{proof}[\bf Proof] Let $x^0=(x_1^0,x_2^0,...,x_N^0)\in\R^{3N}$ be any point with $x_1^0=x_2^0$. For $r>0$ and $\eps>0$, define the following neighborhood of this point:  
		$$
		B'_\eps \times B'_r \times (B_r)^{N-2} := \{x\in\R^{3N} \, : \, |\tfrac{x_1-x_2}{\sqrt{2}}|\le \eps, \, |\tfrac{(x_1-x_1^0)+(x_2-x_2^0)}{\sqrt{2}}|\le r, \, |x_i-x_i^0|\le r \, \forall i=3,...,N\}. 
		$$
		Now assume for contradiction that there exists an invertible Lipschitz map $T\, : \, \Omega^N \to \R^{3N}$ such that $T_\sharp \rho_0 = \rho_1$. 
		Introduce the quantities 
		$$
		f_{min} := \min_{T^{-1}(B_r'\times B_r'\times (B_r)^{N-2})} \rho_0, \;\;\;\;\;\; 
		\eta(\eps) := \max_{B'_{\eps}\times B'_r \times (B_r)^{N-2}} \rho_1. 
		$$
		Note that the minimum is attained because the set over which it is taken is closed (since $T$ is continuous) and bounded (since $\Omega$ is). 
		The key point is that Lipschitz maps can inflate the Lebesgue measure of any measurable set $A\subset\R^{3N}$ only by a factor, 
		$
		\calL(T(A)) \le  ({\rm Lip}\, T)^{3N} \calL(A).
		$
		We exploit this as follows: 
		\begin{align*}
			f_{min} \cdot \calL\bigl(B_\eps' \times B_r' \times (B_r)^{N-2}\bigr) 
			&= f_{min} \cdot \calL\bigl(T(T^{-1}(B_\eps' \times B_r' \times (B_r)^{N-2}))\bigr) \\
			&\le f_{min} \cdot (Lip \, T)^{3N} \calL\bigl(T^{-1}(B_\eps' \times B_r' \times (B_r)^{N-2})\bigr) \\
			&\le (Lip \, T)^{3N} \mu\bigl(T^{-1}(B_\eps' \times B_r' \times (B_r)^{N-2})\bigr) \\
			&= (Lip \, T)^{3N} \nu\bigl(B_\eps' \times B_r' \times (B_r)^{N-2}\bigr) \\
			& = (Lip \, T)^{3N} \nu\bigl(B_\eps' \times B_r' \times (B_r)^{N-2}\bigr) \\
			& \le (Lip \, T)^{3N} \calL\bigl(B_\eps' \times B_r' \times (B_r)^{N-2}\bigr) \cdot \eta(\eps). 
		\end{align*}
		But $f_{min}$ is positive whereas $\eta(\eps)$ tends to zero as $\eps\to 0$, a contradiction. 
	\end{proof}

	\section{Normalizing flow approximation of transport maps}
	Throughout this section, we work in a general dimension $d$, which in the molecular dynamics setting corresponds to $d = 3N$. 
	
	Further, we assume that the map has the form of an $\ell$-fold composition
	\begin{equation} \label{comp_struc}
		T = G^{(\ell)} \circ \cdots \circ G^{(1)}.
	\end{equation}
	This composition structure emerges naturally for 
	(i) the Brenier transport map which can be constructed via the Benamou--Brenier dynamic formulation of optimal transport, 
	(ii) the Moser transport map \cite{moser1965,dacorogna1990} which is based on the flow of a more explicit vector field; see section \ref{sec:comp_struc}.
	
	In practice, the underlying flow and its step maps $G^{(k)}$ in \eqref{comp_struc} 
	are not available explicitly. 
	
	The normalizing flow approach approximates each $G^{(k)}$ by a trainable invertible neural
	network block $F^{(k)}_{\theta_k}$,
	\[
	G^{(k)} \approx F^{(k)}_{\theta_k},
	\]
	(which itself consists of multiple coupling layers and alternating input partitions). 
	The overall approximated transport map is then given by
	\[
	F_\theta
	=
	F_{\theta_\ell}^{(\ell)} \circ
	F_{\theta_{\ell-1}}^{(\ell-1)} \circ
	\cdots \circ
	F_{\theta_1}^{(1)},
	\]
	forming a neural-network discretization of the continuous transport.
	The resulting density is
	\[
	\rho_\theta = (F_\theta)_\sharp \rho_0,
	\]
	which approximates the target density $\rho_1$.
	Each $F^{(k)}_{\theta_k}$ is designed to be bijective with analytically
	computable inverse and Jacobian determinant.
	
	The RealNVP architecture \cite{dinh2017density} provides a widely used
	framework for constructing such invertible transformations.
	For a RealNVP flow with $\ell$ coupling layers, the transformation reads
	\[
	z = F_\theta(x)
	=
	F^{(\ell)} \circ F^{(\ell-1)} \circ \cdots \circ F^{(1)}(x).
	\]
	
	For each layer $k = 1, \dots, \ell$, the forward transformation is
	\[
	\begin{cases}
		z^{(k)}_a = x^{(k-1)}_a, \\[4pt]
		z^{(k)}_b
		=
		x^{(k-1)}_b \odot
		\exp\!\big(s^{(k)}_{\theta_k}(x^{(k-1)}_a)\big)
		+
		t^{(k)}_{\theta_k}(x^{(k-1)}_a),
	\end{cases}
	\]
	with $x^{(0)} = x$ and $z = z^{(\ell)}$.
	The inverse mapping is explicit:
	\[
	\begin{cases}
		x^{(k-1)}_a = z^{(k)}_a, \\[4pt]
		x^{(k-1)}_b
		=
		\big(z^{(k)}_b - t^{(k)}_{\theta_k}(z^{(k)}_a)\big)
		\odot
		\exp\!\big(-s^{(k)}_{\theta_k}(z^{(k)}_a)\big).
	\end{cases}
	\]
	
	The Jacobian matrix of each layer is block-triangular:
	\[
	J_{F^{(k)}}(x^{(k-1)})
	=
	\begin{pmatrix}
		I_{d_a} & 0 \\
		\dfrac{\partial z^{(k)}_b}{\partial x^{(k-1)}_a}
		&
		\operatorname{diag}\!\big(\exp(s^{(k)}_{\theta_k}(x^{(k-1)}_a))\big)
	\end{pmatrix},
	\]
	with determinant
	\[
	\det J_{F^{(k)}}(x^{(k-1)})
	=
	\exp\!\left(
	\sum_{j=1}^{d_b}
	s^{(k)}_{\theta_k,j}(x^{(k-1)}_a)
	\right).
	\]
	
	For the complete flow, the Jacobian determinant factorizes as
	\[
	\det J_{F_\theta}(x)
	=
	\prod_{k=1}^{\ell}
	\det J_{F^{(k)}}(x^{(k-1)}).
	\]
	This enables efficient density evaluation via the change-of-variables
	formula
	\[
	\log \rho_\theta(z)
	=
	\log \rho_0(F_\theta^{-1}(z))
	-
	\log \big|\det J_{F_\theta}(F_\theta^{-1}(z))\big|.
	\]
	
	\subsection{Composition structure of the Brenier and Moser maps} \label{sec:comp_struc}
	
	We now describe how the composition structure \eqref{comp_struc} arises naturally for standard transport maps.
	
	We begin with the Brenier map. Its basic static construction is as the optimizer $T\,  : \, \R^d \to \R^d$ of 
	$\int |x - T(x)|^2 \rho_0(x)\, dx$ 
	subject to the constraint $T_\sharp \rho_0 = \rho_1$. Alternatively, it is the time-1 flow map of the time-dependent vector field $(v_t)_{t\in[0,1]}$, $v_t \, : \, \R^d\to\R^d$   which minimizes the action
	$$
	\int_0^1 |v_t(x)|^2 \rho_t(x)\, dx \, dt 
	$$
	over velocity fields and density fields $(\rho_t)_{t\in[0,1]}$ 
	subject to the continuity equation 
	\begin{equation} \label{conteq}
		\frac{\partial \rho_t}{\partial t}
		+
		\nabla \cdot (\rho_t v_t) = 0
	\end{equation}
	and the endpoint conditions $\rho_t|_{t=0}=\rho_0$,$\rho_t|_{t=1}=\rho_1$. The optimal density field is the displacement interpolation $\rho_t = \Phi_t{}_\sharp \rho_0$, where $\Phi_t$ is the flow map of the optimal vector field. The flow map is  
	$\Phi_t(x)=(1-t)x + tT(x)$, that is to say it moves mass along linear trajectories from a point $x$ to its image $T(x)$ under the Brenier map. In particular, $T=\Phi_1$, that is, it is the time-one map of a continuous flow. 
	
	The Moser map has the advantage that it is more explicit, but the price to pay is that the densities  
	$\rho_0$ and $\rho_1$ must be assumed to be sufficiently smooth and strictly positive on a bounded domain with smooth boundary. The construction goes as follows (see \cite{moser1965, dacorogna1990}, and see \cite{friesecke2024optimal} for a textbook account). Define the velocity field 
	\begin{equation} \label{vt}
		v_t(x) = \frac{\nabla u(x)}{(1-t)\rho_0(x) + t \rho_1(x)},
	\end{equation}
	where $u$ solves the following Poisson equation with Neumann boundary conditions
	\begin{equation} \label{Neumann}
		-\Delta u = \rho_1 - \rho_0 \text{ in }\Omega, \;\;\;
		\nabla u \cdot n = 0 \text{ on }\partial \Omega. 
	\end{equation}
	Here $n(x)$ is the outward unit normal to $\partial \Omega$ at $x$. This construction guarantees that the linear density interpolation $\rho_t = (1-t)\rho_0 + t\rho_1$ satisfies the continuity equation \eqref{conteq}. 
	The velocity field $v_t$ gives rise to a flow map $\Phi_t : \Omega \to \Omega$, by solving the ordinary differential equation
	\begin{equation} \label{ODE}
		\frac{d}{dt}\Phi_t(x) = v_t(\Phi_t(x)),
		\qquad
		\Phi_0(x) = x.
	\end{equation}
	The map $T=\Phi_1$ is then a smooth diffeomorphism satisfying $(\Phi_1)_\sharp \mu = \nu$.
	
	Finally, for both the Brenier map and the Moser map, the composition structure \eqref{comp_struc} is obtained by discretizing the time interval $[0,1]$
	into $\ell$ uniform subintervals, each of length $1/\ell$, and introducing the flow maps $\Phi_{s,t}$ by 
	replacing the initial condition in \eqref{ODE} by $\Phi_{s,t}(x)|_{t=s} = x$. The flow map over the $k$-th timestep is then, letting $t_k=k/\ell$, 
	\[
	G^{(k)}(x) = \Phi_{t_{k-1},t_k}(x)
	\approx
	x + \frac{1}{\ell} v_{t_k}(x),
	\qquad k = 1, \dots, \ell,
	\]
	providing the decomposition 
	\[
	T
	= 
	G^{(\ell)} \circ G^{(\ell-1)} \circ \cdots \circ G^{(1)}.
	\]
	
	\section{Theoretical results}
	In the following section, we provide a rigorous theoretical foundation for the Boltzmann generator method \cite{noe2019boltzmann}, by combining a regularization argument, a rigorous construction of Moser transport for low-regularity endpoint densities, and known approximation results for RealNVP flows.

	\subsection{Regularization}
	Because of the nonexistence of smooth transport maps revealed by Proposition \ref{P:nonexist}, we begin by introducing a regularization of the potential such that the error in the generated samples can be made arbitrarily small, depending on the choice of the regularization parameter.

	\begin{theorem}[Convergence of regularized Boltzmann distributions] \label{thm:pot_approx} For an arbitrary number $N$ of atoms, consider the configuration space $\Omega$, where $\Omega\subset\R^{3N}$ is open and bounded. Let $\beta>0$, and let 
		$U : \Omega \to \R\cup\{+\infty\}$ be a potential satisfing \eqref{pot}. \\[2mm]
		{\rm a)} There exists a family of regularized potentials $\{U_\varepsilon\}_{\varepsilon > 0}$, $U_\varepsilon : \Omega \to \R$, satisfying:
		\begin{enumerate}
			\item[(i)] \textbf{Regularity:} $U_\epsilon$ is $C^1$ and Lipschitz.
			\item[(ii)] \textbf{Uniform lower bound:} There exists $M >0$ such that $U_\varepsilon(x) \ge -M$ for all $x \in \Omega$ and all $\varepsilon > 0$.
			\item[(iii)] \textbf{Pointwise convergence:} For all $x \in \Omega$, $\lim_{\varepsilon \to 0} U_\varepsilon(x) = U(x)$. 
			\item[(iv)]
			\textbf{No change except at high energy:} $U_\eps(x)=U(x)$ for all $x$ with $U(x)\le \tfrac{1}{\eps}$.  
		\end{enumerate}
		{\rm b)} 
		Define the normalized Boltzmann densities
		\[
		\rho_\varepsilon(x) = \frac{F_\varepsilon(x)}{Z_\varepsilon}, 
		\qquad 
		\rho(x) = \frac{F(x)}{Z}, 
		\]
		with the partition functions
		\[
		Z_\varepsilon = \int_\Omega F_\varepsilon(x)\,dx, 
		\qquad 
		Z = \int_\Omega F(x)\,dx.
		\]
		Then $\rho_\eps$ is $C^1$ and Lipschitz, and 
		\[
		\lim_{\varepsilon \to 0} \|\rho_\varepsilon - \rho\|_{L^1(\Omega)} = 0.
		\]
	\end{theorem}
	A important feature of the above regularization is property (iv). Thus the potential, the molecular dynamics trajectories, and the Boltzmann distribution are unchanged except in regions of very high energy. At room temperature, trajectories will in practice not reach such energies and the above regularization is only needed for rigorous theory but not for practical learning of normalized flows.  
	\begin{proof}[\bf Proof] For $\lambda>0$, define the set $K_\lambda = \{x\in \Omega\, : \, U(x)>\lambda\}$. Choose a cutoff function $\zeta\in C^1(\R^{3N})$ with $0\le \zeta \le 1$ which is $1$ on $K_{2/\eps}$ and $0$ outside $K_{1/\eps}$, and let $U_\eps(x)=(1-\zeta(x))U(x)+\zeta(x) \cdot \tfrac{2}{\eps}$.  Then (i) and (iv) are  satisfied. Moreover $U_\eps$ is everywhere finite, $C^1$, and globally Lipschitz, and so is $\rho_\eps$. 
		Since the $U_\varepsilon$ are uniformly bounded below and converge pointwise to $U$, the Boltzmann factors $F_\varepsilon = e^{-\beta U_\varepsilon}$ converge pointwise to $F = e^{-\beta U}$. The bound $F_\varepsilon \le e^{\beta M}$ allows application of the dominated convergence theorem, implying $Z_\varepsilon \to Z$ and $F_\varepsilon \to F$ in $L^1(\Omega)$. 
		Hence $\rho_\varepsilon = F_\varepsilon / Z_\varepsilon \to \rho = F / Z$ in $L^1(\Omega)$, completing the argument.
	\end{proof}
	
	\subsection{Moser transport map: existence and approximation by RealNVP flow}
	
	We now prove that the Moser transport map exists between low-regularity endpoint densities, thereby extending the classical construction in a smooth setting \cite{moser1965, dacorogna1990}. We then show Sobolev approximability of the map by invertible neural networks. 
	
	\begin{theorem}[Existence and RealNVP approximation of the Moser transport map] \label{thm:moser-approx}
		Let $\Omega \subset \mathbb{R}^d$ be a bounded $C^3$ domain. Let $\rho_0$ and $\rho_1$ be Lipschitz continuous probability densities on $\Omega$ which are strictly positive with $\inf_\Omega \rho_0>0$, $\inf_\Omega\rho_1 >0$.   
		\\[1mm]
		{\rm a)} There exists a solution $u\in W^{2,2}(\Omega)$ to the Neumann problem \eqref{Neumann}, and this solution belongs to $W^{3,p}(\Omega)$ for all $p>1$. The corresponding Moser vector field \eqref{vt} is Lipschitz. The associated Moser transport map $T=\Phi_1$ defined by \eqref{ODE} is an invertible bilipschitz map  and transports $\rho_0$ to $\rho_1$, that is to say it satisfies \eqref{pfw}. 
		\\[2mm]
		If in addition $\rho_0$ and $\rho_1$ are $C^1$, then: 
		\\[2mm]
		{\rm b)} The Moser transport map $T=\Phi_1$ 
		is a $C^1$-diffeomorphism.
		\\[2mm]
		{\rm c)} For every $p>d$ and every $\varepsilon>0$, there exists an invertible
		RealNVP neural network $F_\theta:\Omega \to \Omega$ such that
		\[
		\|F_\theta - T\|_{W^{1,p}(\Omega)} < \varepsilon.
		\]
	\end{theorem}
	\vspace*{0mm}
	
	\noindent
	\begin{proof}[\bf Proof] We begin with a). Existence of a $W^{2,2}$ solution to the Neumann problem \eqref{Neumann} is a standard fact of PDE theory, noting that $\rho_1-\rho_0$ belongs to $L^2(\Omega)$ and integrates to $0$. Higher $W^{3,p}$ regularity is proved in  \cite{grisvard2011} Lemma 2.4.2.2 and Theorem 2.5.1.1, noting that $\rho_1-\rho_0$ belongs to $W^{1,p}(\Omega)$. This higher regularity means $\nabla u \in W^{2,p}$, and the Sobolev embedding theorem for $p>d$ gives $\nabla u\in C^1(\overline{\Omega})$. In particular, $\nabla u$ is Lipschitz. Together with the assumptions on $\rho_0$ and $\rho_1$ this yields that the vector field $v_t$ defined in \eqref{vt} is Lipschitz in $x$ and continuous in $t$. It is then a standard fact of ODE theory that the associated flow map $\Phi_t\, : \, \Omega\to\Omega$ exists and is an invertible bilipschitz map for any $t$. By construction $T$
		transports $\rho_0$ to $\rho_1$. 
		b) follows from the fact that under the additional assumption made,  the vector field \eqref{vt} is $C^1$. 
		As regards c), RealNVP coupling flows are dense with respect to the $W^{1,p}$ topology in the class of $C^1$-diffeomorphisms
		on compact domains \cite{ishikawa2023universal}. For an earlier, related approximation result in Sobolev spaces (namely approximability of Lipschitz functions by ReLu networks in the $W^{s,p}$ topology for all $s<1$) see 
		\cite{guhring2020error}. 
		Thus $T$ can be approximated arbitrarily well in $W^{1,p}(\Omega)$
		by a RealNVP map $F_\theta$.
		Finally, being a RealNVP map, $F_\theta$ is invertible. 
	\end{proof}

	\subsection{Convergence of push-forward density}
	Next we prove convergence of the push-forward density under the RealNVP flow, in the narrow sense. This relies on the $W^{1,p}$ convergence result from Theorem \ref{thm:moser-approx}, the Sobolev embedding theorem, and simple estimates. 
	
	Recall that a sequence of probability distributions $\rho_j$ on any closed subset of $\R^d$ converges narrowly to a probability distribution $\rho$ on $\Omega$ if
	\begin{equation} \label{narrow}
		\int_{\Omega} f(y)d\rho_j(y) \to \int_{\Omega} f(y) d\rho(y) \mbox{ for all }f\in C_b(\Omega),
	\end{equation}
	where $C_b(\Omega)$ is the space of bounded continuous functions on $\Omega$. 
	\begin{theorem}[Convergence of push-forward densities]
		\label{thm:density_convergence}
		Let $\Omega$, $\rho_0$, $\rho_1$ be as in Theorem \ref{thm:moser-approx} {\rm b)}. 
		Then there exists a sequence of invertible realNVP neural networks $F_{\theta_j}\, : \, \Omega\to\Omega$ such that
		$$
		F_{\theta_j}{}_\sharp \rho_0 \to \rho_1 \text{ narrowly as }j\to\infty.
		$$
	\end{theorem}
	The theorem says that bounded continuous observables are approximated correctly by the RealNVP coupling flow. 
	
	We conjecture that such an approximation result is also possible in the strong $L^1$ sense. This would follow if the RealNVP map $F_\theta$ from Theorem \ref{thm:moser-approx} could be constructed in such a way that, in addition, its inverse $F_\theta^{-1}$ approximates $T^{-1}$ arbitrarily well in $W^{1,p}(\Omega)$. Such a result with $p=d$ would imply that the inverse Jacobian 
	appearing in \eqref{pfw_densities} is approximated in $L^1(\Omega)$. 
	\\[2mm]
	\begin{proof}[\bf Proof] Let $T \, : \, \Omega\to\Omega$ be the Moser map between $\rho_0$ and $\rho_1$ from Theorem \ref{thm:moser-approx} a), b). By Theorem \ref{thm:moser-approx} c) and the Sobolev embedding $W^{1,p}(\Omega)\hookrightarrow C^0(\overline{\Omega})$ for $p>d$, there exists a sequence of invertible RealNVP neural network $F_{\theta_j}$  such that 
		\begin{equation} \label{unifconv}
			\sup_{x\in\Omega} |F_{\theta_j}(x)-T(x)| \to 0 \; (j\to\infty).
		\end{equation}
		For any given $f\in C_b(\Omega)$ we now calculate using the change-of-variables formula for the push-forward
		\begin{align*}  
			\Big|\int f \, d F_{\theta_j}{}_\sharp \rho_0 - \int f \, d\rho_1 \Big| &=  \Big|\int f \, d F_{\theta_j}{}_\sharp \rho_0 - \int f \, dT_\sharp\rho_0 \Big| \\ 
			&= \Big| \int f(F_{\theta_j}(x)) d\rho_0(x) - \int f(T(x)) \, d\rho_0(x) \Big| \\
			&\le \sup_{x\in\Omega} \big|f(F_{\theta_j}(x)) - f(T(x))\big| \; \int 1 \, d\rho_0 \; \to 0 \; (j\to\infty),
		\end{align*}
		where the convergence in the last line is due to the fact the first factor converges to zero thanks to \eqref{unifconv} and the continuity of $f$, and the second factor is equal to $1$ since $\rho_0$ is a probability measure. Since $f$ was arbitrary, this establishes the asserted narrow convergence \eqref{narrow}. 
	\end{proof}
	
	\subsection{Main result} 
	Finally we are in a position to state our main result. 
	\begin{theorem} [Theoretical justification of the Boltzmann generator method] \label{thm:Wasserstein} 
		For an arbitrary number $N$ of atoms, any bounded domain $\Omega\subset\R^{3N}$ of class $C^3$, any inverse temperature $\beta>0$, any potential 
		$U : \Omega \to \R\cup\{+\infty\}$ satisfing \eqref{pot}, any $C^1$ and Lipschitz reference probability density on $\Omega$ which is strictly positive with $\inf_\Omega \rho_0>0$, and any $\eps>0$, there exists an invertible
		RealNVP neural network $F_\theta:\Omega \to \Omega$ such that
		\[
		W_2(F_\theta{}_\sharp \rho_0, \, \rho_1) < \varepsilon,
		\]
		where $\rho_1$ is the Boltzmann distribution \eqref{boltz} and $W_2$ denotes the Wasserstein-2 distance. 
	\end{theorem}
	
	\noindent
	\begin{proof}[\bf Proof] This follows by combining Theorems \ref{thm:pot_approx}, \ref{thm:moser-approx}, \ref{thm:density_convergence} and the fact (see e.g. \cite{villani2003,  friesecke2024optimal}) that the Wasserstein-2 distance metrizes narrow convergence of probability measures on compact sets. 
	\end{proof}
	
	\section{Numerical Experiments}
	
	In this section, we illustrate the Boltzmann generator method for two models: (1) overdamped Langevin dynamics for a particle in a two-dimensional double-well potential, (2) full molecluar dynamics for the alanine dipeptide molecule. We also check  how close the true and generated distributions are in the Wasserstein-2 distance and find close agreement, in line with what is expected from our theoretical findings (Theorem \ref{thm:Wasserstein}).

	\subsection{2D Double-well potential}
	\label{sec:double-well}
	We simulated the Boltzmann distribution \eqref{boltz} for a 2D double-well potential 
	using overdamped Langevin dynamics
	\[
	dx = - \nabla U(x) \, dt  + \sqrt{\tfrac{2}{\beta}} dW
	\]
	to obtain stochastic trajectories whose position distribution approaches the Boltzmann distribution at long time.   
	The potential is defined as
	\[
	U(x_1, x_2) = \frac{1}{4}(x_1^2 - 1)^2 + \frac{1}{2}x_2^2,
	\]
	which forms two stable wells along the horizontal axis. The random thermal fluctuations allow the particle to cross the energy barrier, producing transitions between metastable basins. This 
	behavior is prototypical for real molecules. For the real data shown in Figure \ref{fig:2Dexample} we used 600 timesteps. 
	
	We trained a RealNVP flow with 6 coupling layers and 128 hidden units, resulting in a total of 7,704 trainable parameters, to learn the equilibrium distribution from Langevin trajectories. As a loss function we used negative log likelihood. After training, we generated 10,000 independent samples from the flow.
	
	\begin{figure}[H]
		\centering
		\begin{subfigure}{0.48\textwidth}
			\centering
			\includegraphics[width=\linewidth]{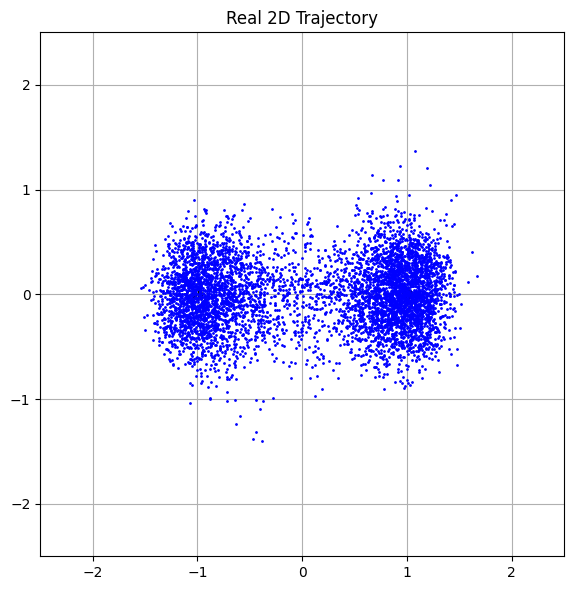}
			\caption{Real Langevin 2D trajectory}
		\end{subfigure}
		\hfill
		\begin{subfigure}{0.48\textwidth}
			\centering
			\includegraphics[width=\linewidth]{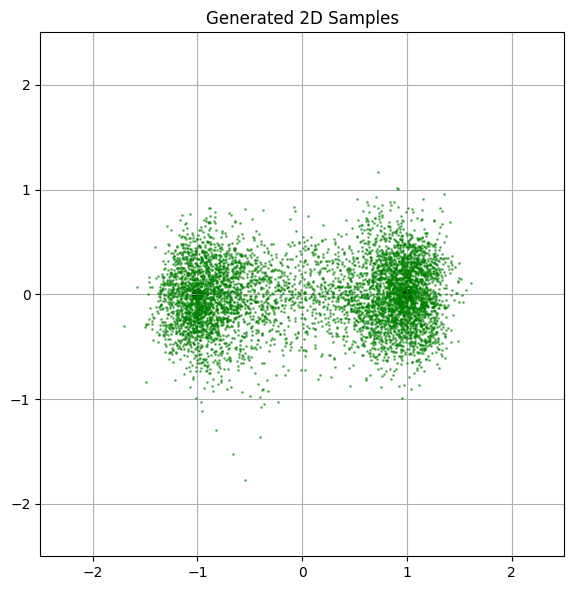}
			\caption{RealNVP generated 2D samples}
		\end{subfigure}
		
		\vspace{0.3cm}
		
		\begin{subfigure}{0.48\textwidth}
			\centering
			\includegraphics[width=\linewidth]{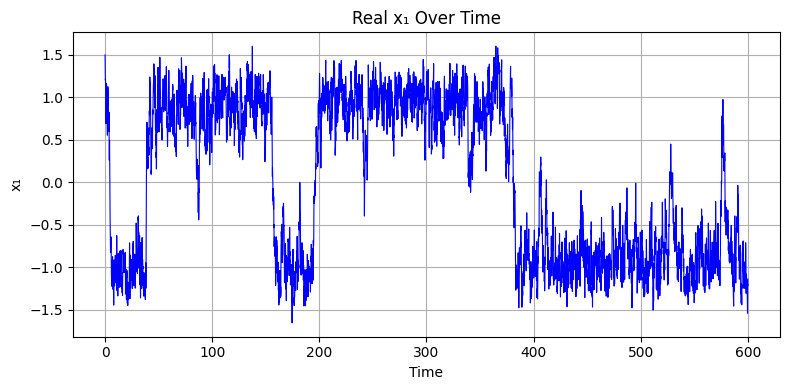}
			\caption{Real $x_1$ component over time}
		\end{subfigure}
		\hfill
		\begin{subfigure}{0.48\textwidth}
			\centering
			\includegraphics[width=\linewidth]{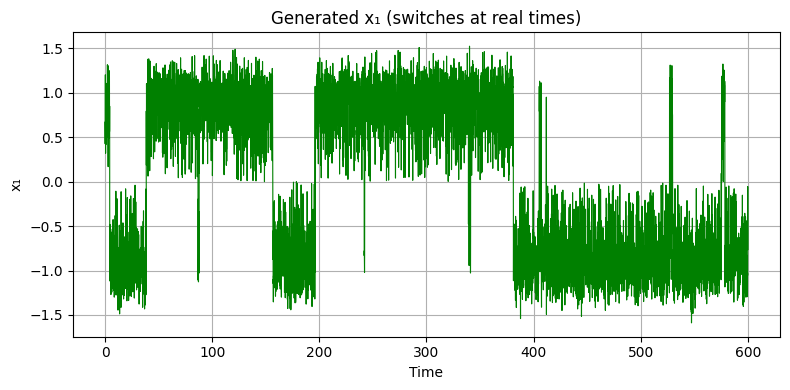}
			\caption{Generated $x_1$ component over time}
		\end{subfigure}
		\caption{Comparison of real Langevin dynamics and RealNVP generated samples for the 2D double-well potential.}
		\label{fig:2Dexample}
	\end{figure}
	
	To evaluate the accuracy of the trained model, we computed the Wasserstein-2 distance $W_2$ between the samples generated by the RealNVP and the reference Langevin trajectories. We used the Python Optimal Transport (POT) library \cite{flamary2021pot}to compute the 2D Wasserstein distance. The 2D joint distribution yields a $W_2$ distance of $0.0993$, while the per-coordinate distances are $0.0101$ for $x_1$ and $0.0036$ for $x_2$, with an average of $0.0069$. 
	
	The model also exactly reproduces the switching dynamics, with both real and generated trajectories exhibiting exactly 41 transitions between the two wells. These results demonstrate that the RealNVP does not just successfully capture the equilibrium Boltzmann distribution, but also the metastable dynamics of the double-well system.
	\subsection{Alanine Dipeptide Molecular System}
	\label{sec:alanine}
	We validated the Boltzmann generator method on realistic molecular data by applying it  to alanine dipeptide, a standard benchmark system for studying conformational dynamics \cite{chodera2007automatic}.
	
	We used a data set with 6 independent 10 ns molecular dynamics (MD) trajectories (60 ns, total frames 59,994) simulated in an implicit solvent at $T=300$ K using the AMBER99SB-ILDN force field  \cite{lindorfflarsen2010}. The backbone dihedral angles, $\Phi$ and $\Psi$, which range over $[-\pi,\pi]$, serve as the slow collective variables of the system and capture the primary conformational changes. The empirical free energy surface $F(\Phi,\Psi)$ is estimated from the long MD trajectories. 
	
	We trained a RealNVP flow on dihedral angle samples projected to $(\Phi,\Psi)$ space to learn the equilibrium densities. We used $\ell=6$ coupling layers with 64 hidden units for each $(s_\theta,t_\theta)$ subnetwork, resulting in 2,316 trainable parameters in total. Again we used negative log likelihood as a loss. The  model successfully captured the main stable conformations of alanine dipeptide, corresponding to the well-known metastable regions of the $(\Phi,\Psi)$ free energy landscape.
	
	\begin{figure}[H]
		\centering
		\begin{subfigure}{0.48\textwidth}
			\centering
			\includegraphics[width=\linewidth]{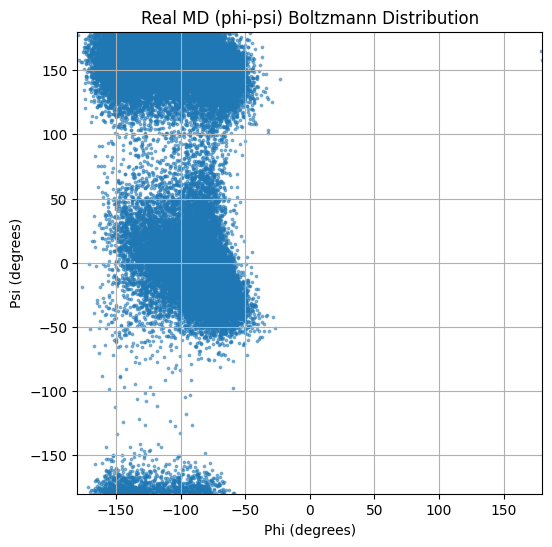}
			\caption{True Boltzmann distribution}
		\end{subfigure}
		\hfill
		\begin{subfigure}{0.48\textwidth}
			\centering
			\includegraphics[width=\linewidth]{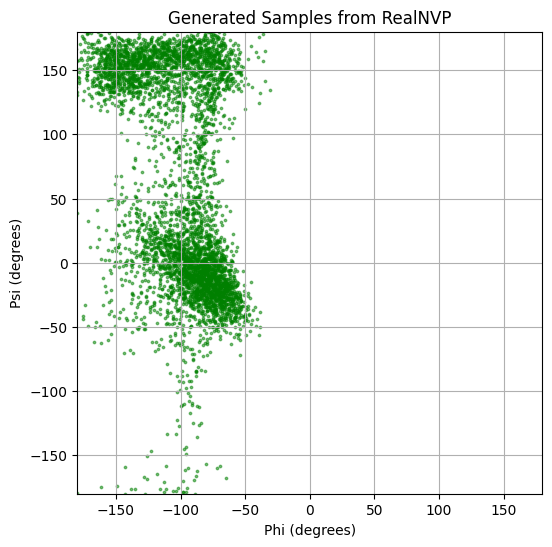}
			\caption{RealNVP generated distribution}
		\end{subfigure}
		\caption{Comparison of true and generated $\Phi$--$\Psi$ distributions for alanine dipeptide}
	\end{figure}
	The generated $(\Phi,\Psi)$ samples closely follow the true Boltzmann distribution obtained from the MD simulation. The main metastable regions appear at similar positions in both plots, indicating that the RealNVP successfully captures the preferred conformations. The small differences seen near transition regions are expected due to sampling noise.
	
	To quantitatively assess the accuracy of the generated samples, as in the previous example we computed the Wasserstein-2 distance between the true and generated distributions using the Python optimal transport library \cite{flamary2021pot}. Following standard practice, we subsampled 1000 points from each distribution. The resulting distance of $9.21^\circ$ confirms very good agreement between the generated and true Boltzmann distributions.

	\section{Discussion and Conclusions}
	
	We have shown that RealNVP-based normalizing flows are capable in theory of producing correct samples from Boltmann distributions of molecules despite their low regularity. Moreover our numerical simulations add further confirmation that these flows reproduce key features of such distributions in practice as first observed in \cite{noe2019boltzmann}, and quantify the error between the generated and true distributions via the Wasserstein distance. 
	
	Our theoretical work shows that a transport map with very nice theoretical properties, e.g. simultaneously controlled size of gradients and inverse gradients, is provided by a simplified version of optimal transport, namely Moser transport. Current RealNVP training protocols provide more complicated transport maps which are expected to exhibit uncontrollably large Lipschitz constants for the inverse map as discussed in other contexts in \cite{behrmann2021, kirichenko2020}. This calls into question numerical invertibility of the flow, and robustness for unseen  initial data. An interesting question left open by our work is whether good numerical invertibility in high dimensions can be achieved by a training bias towards ``nice'' transport maps like the Moser map without introducing too much additional computational complexity.

\begin{small}
\section*{Acknowledgements}
\end{small}

\noindent
This work was funded by DFG within the priority program SPP 2298 \textit{Theoretical Foundations of Deep Learning}, project number 543965508. 
\vspace*{3mm}

\begin{spacing}{0.9}
\bibliographystyle{alpha}

\newcommand{\etalchar}[1]{$^{#1}$}

\end{spacing}

\end{document}